\documentclass{article}

\usepackage[preprint]{neurips_2024}

\usepackage{txie}

\usepackage[utf8]{inputenc} 
\usepackage[T1]{fontenc}    
\usepackage{hyperref}       
\usepackage{url}            
\usepackage{booktabs}       
\usepackage{amsfonts}       
\usepackage{nicefrac}       
\usepackage{microtype}      
\usepackage{xcolor}         
\usepackage{smile}
\usepackage[scaled=.91]{helvet}
\usepackage{inconsolata}
\usepackage{titlesec}

\title{Self-Play with Adversarial Critic: \\
{\Large Provable and Scalable Offline Alignment for Language Models}}

\author{
  Xiang Ji \\
  Princeton University \\ \small \href{mailto:xiangj@princeton.edu}{\texttt{xiangj@princeton.edu}}
  \And 
  Sanjeev Kulkarni \\
  Princeton University \\ \small \href{mailto:kulkarni@princeton.edu}{\texttt{kulkarni@princeton.edu}}
  \And 
  Mengdi Wang\thanks{Corresponding authors}\\
  Princeton University \\ \small \href{mailto:mengdiw@princeton.edu}{\texttt{mengdiw@princeton.edu}}
  \And 
  Tengyang Xie\footnotemark[1]\\
  UW-Madison \\ \small \href{mailto:tx@cs.wisc.edu}{\texttt{tx@cs.wisc.edu}}
}

\begin{document}

\maketitle

\begin{abstract}
    This work studies the challenge of aligning large language models (LLMs) with offline preference data. We focus on alignment by Reinforcement Learning from Human Feedback (RLHF) in particular. While popular preference optimization methods exhibit good empirical performance in practice, they are not theoretically guaranteed to converge to the optimal policy and can provably fail when the data coverage is sparse by classical offline reinforcement learning (RL) results. On the other hand, a recent line of work has focused on theoretically motivated preference optimization methods with provable guarantees, but these are not computationally efficient for large-scale applications like LLM alignment. To bridge this gap, we propose \alg, a new offline preference optimization method with self-play, inspired by the on-average pessimism technique from the offline RL literature, to be the first provable and scalable approach to LLM alignment. We both provide theoretical analysis for its convergence under single-policy concentrability for the general function approximation setting and demonstrate its competitive empirical performance for LLM alignment on a 7B Mistral model with Open LLM Leaderboard evaluations. 
\end{abstract}

\allowdisplaybreaks

\section{Introduction}

Recent advances in large language models (LLMs) have greatly strengthened their ability at a multitude of high-level, complex instruction-following tasks, such as coding, summarization, and solving mathematical problems \citep{stiennon2020learning,li2023textbooks}. However, there still remains the challenge of ensuring the content generated by LLMs is aligned with human values and preferences \citep{bai2022constitutional,touvron2023llama}, especially in regards to safety, helpfulness, etc. Specifically, the goal is to make LLMs generate an appropriate response upon receiving any prompt. To achieve this, Reinforcement Learning from Human Feedback (RLHF) is predominantly used. Standard RLHF approaches use human feedback to construct a real-valued reward function that quantifies the desirability of a prompt-response pair and optimize LLM parameters so that the model generates in a way maximizing such reward. 

In practice, data collection can be slow and costly and is usually done before the alignment process, so RLHF is commonly formulated as an offline reinforcement learning (RL) problem in the literature. Through the lens of offline RL, common RLHF methods \citep{christiano2017deep,rafailov2024direct,bai2022constitutional} can be viewed as RL methods that first learn the empirically best reward function from the offline data and then optimize on it for a policy. However, despite their great empirical performance, it is unclear whether these popular alignment methods are guaranteed to converge to the optimal policy. In fact, by standard offline RL theory \citep{zhu2023principled,wang2020statistical}, methods that greedily optimize the empirically best reward or value function can provably fail from converging to the optimal policy when the data coverage is sparse, which is the case for LLM alignment in practice. Several works in the alignment literature have also identified the distribution shift in the training data as a factor for reward hacking \citep{rame2024warm,pan2022effects,miao2024mitigating}.

In response to this issue, a few theoretically motivated preference optimization methods have been proposed in a recent line of work. These methods all make use of the pessimism principle from the offline RL literature \citep{jin2021pessimism,li2024settling}. In particular, by estimating a pointwise lower confidence bound for the latent reward function and optimizing the policy with respect to this reward lower bound, these pessimistic algorithms are theoretically guaranteed to converge to the optimal policy even under sparse data coverage, namely, when the offline data only satisfies single-policy concentrability. Only requiring the offline data to cover where an optimal policy can visit, this is one of the weakest data coverage assumption in the offline RL literature. In fact, it is unlikely that stronger coverage assumptions can be satisfied in practice, e.g., assuming the data cover where any possible policy can visit, or even assuming full coverage of the entire state-action space, which methods based on the empirically best estimator requires to converge provably. Under this modest data assumption, \citet{zhu2023principled} and \citet{xiong2023gibbs} study the linear setting, but their algorithms and theoretical guarantees heavily rely on the linear assumption and are not directly applicable to the general setting. \citet{zhan2024provable} propose {\sf FREEHAND} that handles offline preference optimization for general function approximation, but the algorithm is computationally hard as the authors noted. The computational difficulty for finding the confidence bound or region under general function approximation is well-known; while these works pioneer the study of provable methods for offline preference optimization, there remains a considerable gap between them and large-scale applications like LLM alignment.

In this work, we propose a new approach for offline preference optimization, named \algfull (\alg), that is not only \textbf{theoretically guaranteed} to converge to the optimal policy under the single-policy concentrability assumption but also \textbf{scalable} to general function classes like LLMs \textbf{at the same time}. Our approach is inspired by \citet{xie2021bellmanconsistent,cheng2022adversarially} from the offline RL literature: we formulate offline preference optimization as a Stackelberg game, in which there exists a dueling dynamics between the learner that optimizes policy with respect to a pessimistic reward estimate and the critic that maintains the pessimistic reward under the learner’s policy. Specifically, instead of estimating a pointwise lower bound on the latent reward, our method maintains an on-average pessimistic reward estimate, i.e., a lower bound on the reward expectation under the learner's policy. Such on-average pessimism is appealing because it allows for this particular Stackelberg game to be solved by a single-timescale iterative \textbf{self-play} algorithm, when coupled with a change-of-variable trick similar to the one in Direct Policy Optimization ({\sf DPO}) \citep{rafailov2024direct}. In addition, \alg can be easily implemented on top of an existing RLHF codebase due to its close connection to {\sf DPO}. In the following sections, we not only present the theoretical result showing \alg converges to the optimal policy under single-policy concentrability with high probability, but also put \alg to test in a 7B LLM alignment experiment and demonstrate the competitive performance of \alg against popular alignment baselines in practice.

In summary, our work makes the following contributions:
\begin{enumerate}[label=(\roman*)]
    \item We propose the \textbf{first} offline preference optimization algorithm that both provably converge to a near-optimal policy under only single-policy concentrability and is computationally scalable to general function approximation at the same time, and provide a non-asymptotic analysis on its suboptimality;

    \item We propose a practical implementation of our algorithm, turning our double-timescale algorithm into a single-timescale direct preference optimization algorithm, which is easy to implement based on existing RLHF codebases;

    \item We confirm the competitive performance of \alg by performing alignment on 7B LLM and evaluating the finetuned models on the widely-used Open LLM Benchmarks against popular baselines.
\end{enumerate}

\subsection*{Related Work}

\paragraph{LLM alignment.} Typically, the process of LLM alignment is comprised of two stages: supervised fine-tuning (SFT) and RLHF. SFT usually precedes RLHF and aims to align LLM with a training set containing human demonstration responses, usually by causal language modeling \citep{mishra2021cross,wei2021finetuned}. The standard RLHF approach \citep{christiano2017deep} involves learning a reward function from a set of human preference data and optimizes the LLM with respect to it using RL methods such as PPO \citep{schulman2017proximal}. Under this pipeline, lightweight reward maximization methods such as rejection sampling are also used \citep{touvron2023llama,li2024q,gao2024rebel}, designed to reduce the computation needed to align the policy model given a learned reward function. As a popular alternative, contrastive RLHF methods are also introduced besides the aforementioned {\sf DPO} \citep{hejna2023contrastive,xu2024contrastive}, which learn the policy model directly and obviate the need to learn a separate reward model. 

While we consider the offline setting, in which a fixed training set of preference data is given, there is also a body of literature that studies the online setting for alignment, in which the model can query newly-annotated preference data during the training process \citep{rosset2024direct,dwaracherla2024efficient,calandriello2024human}. However, the requirement of fresh annotations at every iteration is costly in practice.

\paragraph{Self-play.} Self-play describes the type of approaches that searches for the solution to a problem by staging the algorithm to play against itself, usually iteratively. Rooted in game theory \citep{brown:fp1951,Heinrich2015FictitiousSI}, the concept of self-play has found many applications in modern machine learning, e.g. AlphaGo Zero \citep{alphago}, Generative Adversarial Networks (GAN) \citep{goodfellow2014generative}, multi-agent reinforcement learning (MARL) \citep{bansal2017emergent,baker2019emergent}, where the problem is innately multiplayer or can be formulated so.
Self-play is also commonly used in adversarial imitation learning (also known as apprenticeship learning) and inverse reinforcement learning \citep[esp.,][]{abbeel2004apprenticeship,syed2007game,ho2016generative}, which closely aligns with this work. The objective here is to maximize performance in comparison to expert behavior, with the reward function being adversarially selected with respect to this goal.
Recently, self-play has caught increasing attention in the literature for LLM alignment, as some game-theoretic formulations of RLHF and thus self-play methods for them have been proposed, such as Self-Play Preference Optimization ({\sf SPO}) \citep{swamy2024minimaximalist}, Self-Play fIne-tuNing ({\sf SPIN}) \citep{chen2024self}, and Direct Nash Optimization ({\sf DNO}) \citep{rosset2024direct}.

\paragraph{Offline RL.} Offline RL has been a well-studied area, both empirically \citep[e.g.,][]{fujimoto2019off,kumar2020conservative,levine2020offline,fujimoto2021minimalist} and theoretically. 
One key algorithmic insight heavily used in offline RL is the use of adversarial critics or models to establish pessimistic estimates \citep[e.g.,][]{kumar2020conservative,xie2021bellmanconsistent,cheng2022adversarially,bhardwaj2024adversarial}, which is also closely related to this paper.
As for the theoretical side, it is extensively studied in the tabular case \citep{pmlr-v162-shi22c,xie2021policy,yan23efficacy}, where the state and action spaces of the environment are discrete. Notably, a method with minimax-optimal sample complexity is introduced and analyzed in \citet{li2024settling}. Methods with function approximation are also studied. A population line of research considers linear function approximation with a known linear feature \citep{pmlr-v119-duan20b,jin2021pessimism,zanette2021provable}. In efforts to be more realistic, several works study the setting with general function approximation \citep{xie2021bellmanconsistent,cheng2022adversarially,yin2022offline,zhan2022offline}, including neural function approximation \citep{fan2020theoretical,nguyen2021sample}.

\paragraph{Preference-based Reinforcement Learning.} We mainly review the works that study the theory of Preference-based Reinforcement Learning (PbRL) methods. Besides the works for the offline setting we have mentioned earlier, there are existing results about the online PbRL setting, e.g., for the tabular case \citep{novoseller2020dueling,xu2020preference} as well as the linear and general function approximations \citep{pacchiano2021dueling,chen2022human}. Related yet different, dueling bandit is another PbRL setting that has been well-studied \citep{dudik2015contextual,mehta2023kernelized,saha2022efficient}.

\paragraph{Notation} 

For any scalar $a$, we denote the sigmoid function with $\sigma(a) := 1/(1+e^{-a})$. We use $\norm{\cdot}_\infty$ to denote the infinity norm of a real-valued function. For any random sample $Z$ with distribution $P$ and a function $f(\cdot)$ of $Z$, we denote the expectation of $f(Z)$ over $P$ with $\EE_{Z\sim P}[f(Z)]$ or $\EE_{Z}[f(Z)]$. For a set $\cY$, $\Delta(\cY)$ denotes the probability simplex over $\cY$. Finally, we adopt the convention $0/0 = 0$.

\section{Preliminaries}

\paragraph{Offline RL} The RLHF portion for LLM alignment is predominantly formulated as offline RL in the literature, and due to the autoregressive nature of LLM, such offline RL can be reduced to an offline contextual bandit \citep{ouyang2022training}, which we represent with $(\cX, \cY, r^\star, \rho)$. The function $r^\star: \cX \times \cY \to \RR$ represents the latent reward function, which is assumed to be deterministic and unknown. In particular, $r(x,y)$ is the reward of response $y\in \cY$ given prompt $x \in \cX$, where $\cX$ is referred to as the state space and $\cY$ the action space. $\rho$ denotes the initial state distribution of the bandit. A policy $\pi: \cX\to\Delta(\cY)$ specifies a rule for response selection given a prompt, where $\pi(\cdot ~|~ x) \in \Delta(\cY)$ represents the response selection probability vector for prompt $x \in \cX$. The value function of policy $\pi$ at $x$ is defined as
\begin{align*}
    V^\pi(x) := \EE_{y\sim\pi(\cdot | x)}\left[r(x,y)\right].
\end{align*}
The goal is to find a policy, which would be an LLM for human alignment, whose response $y$ maximizes the reward $r(x,y)$ on any prompt $x$. Such a policy is referred to as an optimal policy $\pi^\star$, which maximizes the value function for all states simultaneously.

We evaluate the performance of a policy $\pi$ by the suboptimality metric defined as follow:
\begin{align}\label{eq:suboptimality-def}
    \subopt(\pi) := \EE_{x\sim\rho}\left[V^{\pi^\star}(x) - V^{\pi}(x)\right].
\end{align}
Suboptimality measures the performance difference between the optimal policy $\pi^\star$ and a given policy $\pi$. Naturally, one aims to minimize the suboptimality and find an algorithm whose suboptimality converges to zero as sample size approaches infinity. 

\paragraph{Offline preference optimization} Given an offline preference dataset $\cD$, the standard RLHF procedure consists of two steps \citep{christiano2017deep}: rewarding learning and RL optimization. During reward learning, the goal is to estimate the latent reward from the preference dataset $\cD$. A data point from $\cD$ includes a prompt $x$ and a pair of responses $(y^\w, y^\lo)$. Here, $(y^\w, y^\lo) \sim \pi_{\rref}(\cdot~|~x)$, where $\pi_{\rref}$ is a reference policy after SFT finetuning, and $y^\w$ is preferred to $y^\lo$ given prompt $x$, which is denoted with $y^\w \succ y^\lo ~|~ x$. Such human preference is commonly assumed to follow the Bradley-Terry (BT) model \citep{Bradley1952RankAO} that depends on the latent reward function $r^\star$, which satisfies
\begin{equation}\label{eq:BTL}
    P(y^\w \succ y^\lo ~|~ x) = \frac{\exp(r^\star(x,y^\w))}{\exp(r^\star(x,y^\w)) + \exp(r^\star(x,y^\lo))}.
\end{equation}

This allows one to use maximum-likelihood estimation (MLE) on the the preference dataset $\cD = \{(x_i,y_i^\w, y_i^\lo)\}_{i=1}^{n}$ to estimate the reward function. In particular, one can obtain the empirically best estimated reward function $\widehat{r}$ from minimizing the empirical negative log-likelihood loss over a function class $\cF$:
\begin{equation}\label{eq:PPO-MLE}
    \widehat{r} \leftarrow \min_{f\in\cF}-\frac{1}{n}\sum_{i=1}^n \log \frac{\exp(f(x_i,y_i^\w))}{\exp(f(x_i,y_i^\w)) + \exp(f(x_i,y_i^\lo))}.
\end{equation}

With $\widehat{r}$ learned, the RL optimization step is about finding a policy that maximizes the following objective:
\begin{equation}\label{eq:PPO}
    \widehat{\pi} \leftarrow \max_{\pi\in\Pi} \EE_{x\sim \cD, y\sim\pi(\cdot|x)}\left[\widehat{r}(x,y) - \alpha D_{\mathrm{KL}}\big(\pi(\cdot~|~x)~||~\pi_{\rref}(\cdot~|~x)\big)\right].
\end{equation}
The KL divergence regularization above is meant to control the deviation of the output policy from the reference policy $\pi_{\rref}$, where $\alpha$ is a hyperparameter. This prevents the RL optimization from exploiting any misestimation within the learned reward and causing reward hacking, but at the same time restricts the output policy $\widehat{\pi}$ to the neighborhood of $\pi_{\rref}$ and keeps it from converging to the actual optimal policy. 

Besides the standard two-step approach, \cite{rafailov2024direct} proposed a popular alternative for RLHF named Direct Preference Optimization (DPO). It uses the analytical solution of \cref{eq:PPO} to transform \cref{eq:PPO-MLE} into the following objective:
\begin{equation}\label{eq:DPO}
    \widehat{\pi} \leftarrow \max_{\pi\in\Pi} \EE_{x\sim \cD, y\sim\pi(\cdot|x)}\left[\widehat{r}(x,y) - \alpha D_{\mathrm{KL}}\big(\pi(\cdot~|~x)~||~\pi_{\rref}(\cdot~|~x)\big)\right],
\end{equation}
which effectively combines the RL optimization step with reward learning.

\section{\algfull (\alg)}\label{sec:alg}

Our algorithm is inspired by the philosophy that offline RL can be formulated as a Stackelberg game \citep{von2010market}, which was also studied in \citet{cheng2022adversarially,bhardwaj2024adversarial}. Slightly different from the formulation in \citet{bhardwaj2024adversarial}, where reward samples are available and Bellman error can be computed, we formulate preference optimization as the following bilevel optimization problem:
\begin{align}
    \widehat{\pi} &\in \arg\max_{\pi} \cL_\cD(\widehat{f}, \pi) \label{eq:stackelberg-policy}\\
    \text{s.t.}\ \widehat{f} &\in \arg\min_{f\in\cF} \cL_\cD(f, \pi) + \lambda\cE_\cD(f), \label{eq:stackelberg-critic}
\end{align}
where $\lambda \ge 0$ is a hyperparameter, and 
\begin{align}\label{eq:L-def}
    \cL_\cD(f, \pi) := \frac{1}{n}\sum_{j=1}^{n}f(x_j,y_j) - f(x_j,y_j'),
\end{align}
where $\{x_j\}_{j=1}^n$ are prompts from the preference dataset $\cD$. For every $j\in\{1,\cdots,n\}$, $y_j'$ is a response generated by the reference policy $\pi_{\rref}$ on $x_j$, which can be taken from the preference dataset $\cD$, and $y_j \sim \pi(\cdot ~|~ x_j)$ is a response newly generated from $\pi$ on $x_j$. In addition, we define
\begin{align}\label{eq:E-def}
    \cE_\cD(f) := -\frac{1}{n}\sum_{i=1}^n \log \frac{\exp(f(x_i,y_i^\w))}{\exp(f(x_i,y_i^\w)) + \exp(f(x_i,y_i^\lo))},
\end{align}
which is the negative log likelihood of $f$ on the samples from $\cD$.

In this bilevel optimization formulation, \cref{eq:stackelberg-policy} seeks the learner policy $\pi$ as the leader in the Stackelberg game and attempts to find $\widehat{\pi}$ under which the value predicted with the critic $\widehat{f}$ is maximized; \cref{eq:stackelberg-critic} seeks the critic $f$ in function class $\cF$ as the follower in the Stackelberg game and attempts to find $\widehat{f}$ by performing an on-average pessimistic policy evaluation relative to $\widehat{\pi}$ while maximizing the log likelihood of $\widehat{f}$ with respect to $\cD$. 

However, while there have been existing results that formulate the alignment problem with bilevel optimization \citep{chakraborty2024parl,shen2024principled}, bilevel optimization is computationally difficult to solve in general. Luckily, this bilevel optimization above can be solved with an iterative algorithm. Specifically, we propose an iterative algorithm (\cref{alg:theoretical}) and replace the policy optimization step in \cref{eq:stackelberg-policy} with mirror descent update, which is a common no-regret policy optimization method in the online RL literature \citep{doi:10.1287/moor.1090.0396,agarwal2021theory,cai2020provably,geist2019theory}. Note that the algorithm plays the actor and critic at the same time in a self-play manner. As we will show in later sections, the output policy of \cref{alg:theoretical} provably converges to the optimal policy under the weak single-policy concentrability assumption on the coverage of preference data $\cD$.
\begin{algorithm}[t]
	\caption{\algtfull (\algt) \label{alg:theoretical}}
	\textbf{Initialize:} preference data $\cD$, initial policy $\pi_{\init}$, hyperparameters $\lambda$ and $\eta$, reward function class $\cF$.
 \begin{algorithmic}[1]
    \State $\pi_1 \leftarrow \pi_{\init}$;
	\For{iteration $t = 1, \cdots, T-1$}
        \State With $\cL_\cD$ define in \eqref{eq:L-def} and $\cE_\cD$ in \eqref{eq:E-def},
        \begin{center}
            $f_{t} \leftarrow \arg\min_{f\in \cF} \cL_\cD(f, \pi_t) + \lambda \cE_\cD(f);$ 
        \end{center}\label{alg-line:theory-line1}

        \State $\forall (x,y)\in\cX\times\cY$, \begin{center}
            $\pi_{t+1}(y~|~x) \propto \pi_t(y~|~x)\exp(f_t(x,y)/\eta);$
        \end{center} \label{alg-line:theory-line2}
        \EndFor
        
    \noindent\Return $\widehat{\pi} := \mathrm{Unif}(\pi_{[1:T]})$;
\end{algorithmic}
\end{algorithm}

Despite its iterative procedure, \cref{alg:theoretical} still has a few shortcomings: (i) a reward model $f_t$ needs to be kept at every iteration besides the policy model, leading to a higher memory cost; (ii) the exponential weighting in \cref{alg-line:theory-line2} cannot be computed exactly for language models; (iii) its two-timescale update can result in delicate stability in practice. Fortunately, we can further improve the scalability of our algorithm by eliminating these three undesirables. Specifically, since \cref{alg-line:theory-line2} can be written as
\begin{equation}
\label{eq:def_dpo_f}
    f_t(x,y) = \eta\log\frac{\pi_{t+1}(y|x)}{\pi_{t}(y|x)} + \log Z_t(x)
\end{equation}
with $Z_t(x) = \sum_y\pi_t(y~|~x)\exp(f_t(x,y)/\eta)$, we can plug this into \cref{alg-line:theory-line1} similar to the derivation of {\sf DPO} \citep{rafailov2024direct}. This gives us a new algorithm \algfull (\alg, \cref{alg:practical}), which is an equivalent but more practical version of \cref{alg:theoretical}. 

In this practical implementation, we no longer refer to the reward function class $\cF$ explicitly, as the reward function is now implicit and $\cF$ is implied by the set of reward functions that can be expressed with policy from $\Pi$ as in \cref{eq:def_dpo_f}. In addition, the choice of $y_j'$ can be flexible. It is acceptable to use the set of chosen responses $\{y_j^\w\}_{j=1}^n$ or the set of rejected responses $\{y_j^\lo\}_{j=1}^n$ from $\cD$. We can even make use of both chosen responses and rejected responses by computing the log density ratio $\log\frac{\pi(y_j'|x_j)}{\pi_t(y_j'|x_j)}$ in \cref{alg-line:prac-line1} of \cref{alg:practical} by averaging $\log\frac{\pi(y_j^\w|x_j)}{\pi_t(y_j^\w|x_j)}$ and $\log\frac{\pi(y_j^\lo|x_j)}{\pi_t(y_j^\lo|x_j)}$, which is how we implement \alg for our experiments. Furthermore, as the pseudocode shows, \alg is closely related to {\sf DPO} and {\sf SPIN} and thus easy to implement based on any RLHF codebase with {\sf DPO} or {\sf SPIN} already implemented. As shown in the experiment section later, \alg can be shown to outperform state-of-the-art alignment methods on LLMs.
\begin{algorithm}[t]
	\caption{\algfull (\alg) \label{alg:practical}}
	\textbf{Initialize:} preference data $\cD$, initial policy $\pi_{\init}$, hyperparameters $\lambda$ and $\eta$, policy class $\Pi$.
 \begin{algorithmic}[1]
    \State $\pi_1 \leftarrow \pi_{\init}$;
\For{iteration $t = 1, \cdots, T-1$}
    \State $\forall j \in\{1,\cdots,n\}$, generate a response $y_j\sim\pi_t(\cdot~|~ x_j)$ and let $y_j' = y_j^\w$;
    \small
    \begin{align*}
    \pi_{t+1} &\leftarrow \arg\min_{\pi\in\Pi} \frac{1}{n}\sum_{j=1}^{n}\eta\left(\log\frac{\pi(y_j|x_j)}{\pi_t(y_j|x_j)} - \log\frac{\pi(y_j'|x_j)}{\pi_t(y_j'|x_j)}\right) \\ & \qquad\qquad\qquad\qquad 
    - \frac{\lambda}{n}\sum_{i=1}^n \log\sigma\left(\eta\log\frac{\pi(y_i^\w|x_i)}{\pi_t(y_i^\w|x_i)} - \eta\log\frac{\pi(y_i^\lo|x_i)}{\pi_t(y_i^\lo|x_i)}\right);
    \end{align*} \label{alg-line:prac-line1}
    \normalsize
\EndFor

\noindent\Return $\widehat{\pi} := \mathrm{Unif}(\pi_{[1:T]})$;
\end{algorithmic}
\end{algorithm}

\section{Theoretical Guarantees of \alg}

An important advantage of \alg is that it can be mathematically proved to converge to the optimal policy under weak assumption on the offline preference data, which we present in this section. For our theory, we first make a set of standard assumptions on the latent reward function and the function class our algorithm uses.
\begin{assumption}[Realizable, bounded reward]\label{assum:reward}
    Assume $r^\star \in \cF$ and there exists $R > 0$ such that any $f\in\cF$ satisfies $\norm{f}_\infty \le R$.
\end{assumption}
Recall that $\cF$ is the implicit reward function class that is induced from the policy class $\Pi$ in \cref{alg:practical} by \cref{eq:def_dpo_f}. Note that we assume realizability only for simplicity. Our theory can be modified to the unrealizable setting by considering the bias-variance tradeoff between the approximation error and complexity of $\cF$. Similarly, we assume that $\cF$ is finite but can be exponentially large and use log-cardinality to measure its statistical complexity. Our theory can also extend to continuous function classes by replacing the log-cardinality with appropriate notions of covering number.
\begin{assumption}[Coverage]\label{assum:coverage}
    Assume 
    \begin{equation*}
        \Cfrak := \sqrt{\sup_{f\in\cF}\frac{\EE_{\substack{x\sim\rho\\ y_1\sim\pi^\star(\cdot|x),y_2\sim\pi_{\rref}(\cdot|x)}}\Big[r^\star(x,y_1) - r^\star(x,y_2) - f(x,y_1) + f(x,y_2)\Big]^2}{\EE_{\substack{x\sim\rho\\ y_1\sim\pi_{\rref}(\cdot|x),y_2\sim\pi_{\rref}(\cdot|x)}}\Big[\big(r^\star(x,y_1) - r^\star(x,y_2) - f(x,y_1) + f(x,y_2)\big)^2\Big]}} < \infty
    \end{equation*}
\end{assumption}
A unique challenge in offline RL is the distribution mismatch between the sampling distribution of the offline dataset and the visitation distribution for the policy to be learned. It is desirable that an algorithm make less assumption on the data coverage and be able to operate in a broader set of scenarios. There are a few different types of assumption on the offline data sampling distribution in the offline RL literature. The strongest assumes full coverage, i.e., the entire state-action space $\cX\times\cY$ needs to be covered \citep{pmlr-v130-yin21a,xie2021batch}. Another strong assumption requires the support of the data sampling distribution to cover everywhere that any policy can possibly visit \citep{Antos2006LearningNP,JMLR:v9:munos08a}. However, these assumptions are not possible in modern machine learning applications with high-dimensional and continuous spaces, like LLM alignment. In contrast, we consider one of the weakest assumptions called single-policy concentrability, which only requires that the data sampling distribution covers the visitation of the optimal policy over a function space induced by $\cF$. Moreover, our characterization for distribution mismatch is tighter than the standard density ratio for any $\cF$ or $\pi^\star$:
\begin{align*}
    \Cfrak \le \norm{\frac{\pi^\star}{\pi_{\rref}}}_\infty := \sup_{x,y}\frac{\pi^\star(y~|~x)}{\pi_{\rref}(y~|~x)}.
\end{align*}
An important disadvantage of popular preference optimization methods in practice is that they have no provable guarantee for finding the optimal policy unless under full data coverage. This is because these methods simply take the empirically best reward estimator obtained from MLE and optimize the policy with respect to it directly. As classical RL theory suggests \citep{rashidinejad2021bridging,zhu2023principled,wang2020statistical}, one can always find a problem instance in which such policy obtained from optimizing the empirically best value function suffers constant suboptimality with constant probability. As mentioned in the introduction, this is believed to be a factor behind the reward hacking phenomenon in the alignment literature, where policy optimization exploits the parts in the learned reward function that are ill-estimated because they are not sufficiently covered in the training data. In LLM alignment, an extra KL regularization is usually implemented to mitigate this undesirable phenomenon, by confining the policy model to a trustworthy reference policy model throughout optimization. However, while this approach gives acceptable policies in practice, it fundamentally changes the optimization objective and in general keeps the policy model from converging to the actual optimal policy.

Finally, the theoretical guarantee of \alg can be summarized in the following theorem.
\begin{theorem}\label{thm:main}
    Assume the $n$ samples in $\cD$ are i.i.d. Suppose \cref{assum:reward,assum:coverage} hold. Let $c_{\mathrm{KL}} := \EE_{x\sim\rho}\left[D_{\mathrm{KL}}\big(\pi^\star(\cdot|x)~||~\pi_{\init}(\cdot|x)\big)\right]$ and $\kappa := 1/(\inf_{z\in[-R,R]} \sigma'(z))$. Choose any $\delta \in (0,1)$. Suppose $\eta = \Theta(\sqrt{\frac{R^2T}{c_{\mathrm{KL}}}})$ and $\lambda = \Theta(\Cfrak\sqrt{\frac{n}{\kappa^2\log( |\Pi|/\delta)}})$. Then with probability at least $1-\delta$, the output policy $\widehat{\pi}$ of \alg achieves
    \begin{align*}
        \subopt(\widehat{\pi}) \le \Ocal\left(\Cfrak\kappa\sqrt{\frac{1}{n}\log\frac{|\Pi|}{\delta}} + R\sqrt{\frac{1}{n}\log\frac{|\Pi|}{\delta}} + R\sqrt{\frac{c_{\mathrm{KL}}}{T}}\right)
    \end{align*}
    for an absolute constant $C_0 > 0$.
\end{theorem}

\paragraph{Proof sketch} The suboptimality of {\sf SPAC} mainly comes from two sources: the statistical error in learning with $\cD$ (\cref{alg-line:theory-line1} of \cref{alg:theoretical}) and the optimization error in the our mirror descent update (\cref{alg-line:theory-line2} of \cref{alg:theoretical}). We can decompose the suboptimality into four terms as follows:
\begin{align*}
    \subopt(\widehat{\pi}) &= \frac{1}{T}\sum_{t=1}^T\Big(\underbrace{\EE_{\rho\times\pi^\star}\left[r^\star(x,y) - f_{t}(x,y)\right] - \EE_{\rho\times\pi_{\rref}}\left[r^\star(x,y) - f_{t}(x,y)\right]}_{I_{t,1}}\notag\\ &\quad + \underbrace{\EE_{\rho\times\pi^\star}[f_{t}(x,y)] - \EE_{\rho\times\pi_t}[f_{t}(x,y)]}_{I_{t,2}} + \underbrace{\EE_{\rho\times\pi_t}\left[f_{t}(x,y)\right] - \EE_{\rho\times\pi_{\rref}}[f_{t}(x,y)]}_{I_{t,3}}\notag\\ &\quad - \underbrace{\left(\EE_{\rho\times\pi_t}\left[r^\star(x,y)\right] - \EE_{\rho\times\pi_{\rref}}[r^\star(x,y)]\right)}_{I_{t,4}}\Big).
\end{align*}
The statistical error is associated with $I_{t,1}$, $I_{t,3}$ and $I_{t,4}$. By a change of measure and the mean value theorem, $I_{t,1}$ can be turned into the total variation distance between the Bernoulli distributions corresponding to the true BT model induced by $r^\star$ and a fictitious BT model induced by $f_t$. This can be further bounded by the log likelihood ratio of the two Bernoulli distributions, resulting in 
\begin{equation*}
    I_{t,1} \lesssim \cE_\cD(f_t) - \cE_\cD(r^\star) + \frac{\kappa^2}{n}\log\frac{|\cF|}{\delta}.
\end{equation*}
Furthermore, it can be combined with $I_{t,1}$ and $I_{t,4}$ to recover the critic's optimization objective $\cL_\cD(f, \pi_t) + \lambda \cE_\cD(f)$, which can be further bounded by $\cL_\cD(r^\star, \pi_t) + \lambda \cE_\cD(r^\star)$ due to pessimism. Then we can arrive at the final bound on the statistical error via a concentration argument.

On the other hand, notice the sum of $I_{t,2}$ coincides with the classic notion of cumulative regret in online learning. To bound this, we use a regret analysis similar to the one for Follow the Regularized Leader (FTRL) but using the reverse KL divergence from $\pi_{\init}$ as the regularizer instead, giving 
\begin{equation*}
\sum_{t=1}^{T}\EE_{\rho\times\pi^\star}[f_{t}(x,y)] - \EE_{\rho\times\pi_t}[f_{t}(x,y)] \le \frac{2R^2T}{\eta} + \eta\EE_{x\sim\rho}\Big[D_{\mathrm{KL}}\big(\pi^\star(\cdot~|~x)~||~\pi_{\init}(\cdot~|~x)\big)\Big].
\end{equation*}
This is because the natural policy gradient (NPG) update, \cref{alg-line:theory-line2} in \cref{alg:theoretical}, is automatically satisfied by the implicit reward function defined in \cref{eq:def_dpo_f}. Details of our proof can be found in \cref{sec:appendix-proof}.

\paragraph{Pessimism in offline RL} Pessimism is a common technique in the offline RL literature that is used to estimate a lower bound for the latent reward or value function. It is important to estimate an accurate lower bound, particularly for $(x,y)$-pairs with a small number of observations in the offline dataset, because by concentration, there is a constant probability for all of these samples to be far from their mean, which can severely throw off the estimation of a naive RL algorithm. Thus, pessimism is crucial to the theoretical guarantee for converging to the optimal policy in the offline setting. 

In the offline RL literature, pessimism is widely studied and used in works that consider the tabular setting \citep{rashidinejad2021bridging,li2024settling}, in which the lower confidence bound (LCB) is maintained for every entry of the reward or value function by subtracting an entrywise penalty from the mean estimation. The same philosophy can be seen applied in works that study offline RL under linear function approximation \citep{jin2021pessimism,yin2022near}, in which an elliptical confidence region is maintained instead for the reward or value estimate, and the policy is determined from the worst reward function from this confidence region in a min-max fashion. These pessimism methods share a commonality: they construct a pointwise pessimistic estimate for the entire latent reward function. In fact, the existing provable preference optimization algorithms \citep{zhu2023principled,zhan2024provable,xiong2023gibbs} all follow from this pointwise pessimism approach. While it is theoretically convenient to obtain a pessimistic estimate this way, as \citet{zhan2024provable} point out, algorithms with pointwise pessimism are computationally hard in general, as it is unclear how to construct a confidence region for a neural network or an LLM in practice. In contrast, inspired by a different concept called on-average pessimism from a more recent line of work \citep{xie2021bellmanconsistent,cheng2022adversarially,bhardwaj2024adversarial}, {\sf SPAC} only needs a reward estimate $f_t$ whose expectation under the learner’s policy $\pi$ is a lower bound, i.e., $\EE_{\rho\times\pi}[f_t] \le \EE_{\rho\times\pi}[r^\star]$, whose computation is scalable to any general function including LLMs.

\paragraph{Connections to other alignment methods} 
{\sf SPAC} can be seen as a bridge between {\sf DPO} and {\sf SPIN}. On the one hand, when we only run {\sf SPAC} for one iteration with $\lambda \to \infty$, {\sf SPAC} reduces to {\sf DPO} which optimizes the maximum likelihood only with respect to the preference data $\cD$. On the other hand, when $\lambda$ is taken to $0$, {\sf SPAC} recovers {\sf SPIN} and becomes an adversarial imitation learning algorithm (in this sense, it can also be viewed as an RLHF extension of \citet{syed2007game,ho2016generative}), which prevents overoptimization of the reward model. 
In fact, a similar connection has been discovered in offline RL and imitation learning \citep{cheng2022adversarially,bhardwaj2024adversarial}. However, it is important to note that {\sf SPAC} enjoys more desirable computational properties as the DPO trick (\cref{eq:def_dpo_f}) can take care of no-regret policy optimization completely without the need for extra assumptions.

\section{Experiments}

We demonstrate that \alg achieves competitive performance over a wide range of evaluation benchmarks as it substantially improves the initial SFT model and results in better finetuned models than counterpart LLM alignment methods, even when the training prompts and testing prompts follow different distributions.

\subsection{Experimental Setup}

\paragraph{Model and Data.} In our experiments, we use {\sf mistral-7b-sft-beta} as the base model, which is a finetuned version of Mistral-7B-v0.1 on the Ultrachat200k dataset \citep{ding2023enhancing} by SFT. For the offline preference dataset, we use the latest version of UltraFeedback Binarized \citep{cui2023ultrafeedback}, which is made of a diverse set of 60K prompt-chosen-rejected triplets. The preferences in this dataset are annotated based on the quality of the responses with respect to instruction-following, truthfulness, honesty and helpfulness, judged by {\sf GPT-4}. 

\paragraph{Implementation.} We run \alg for 3 iterations. We shuffle the samples in UltraFeedback Binarized and split them into 3 subsets of equal size to be used in each iteration \citep{rosset2024direct,viethoangtranduong}. As the pseudocode of \cref{alg:practical} lays out, in each iteration $t\in\{0,1,2\}$, we use $\pi_t$ to generate one response $y_j$ to every prompt $x_j$ from the set of 20K prompts designated to that iteration. We initialize with $\pi_t$ and optimize the minimization objective (\cref{alg-line:prac-line1} of \cref{alg:practical}) for two epochs to obtain a new policy model $\pi_{t+1}$. More details about generation and finetuning can be found in \cref{sec:appendix-exp}.

On the other hand, in theory (\cref{thm:main}), a huge scaling $\lambda = \Theta(\Cfrak\sqrt{n})$ is needed on the regularization to guarantee convergence. However, such huge numerics can often lead to numerical instability in optimization. To maintain the same level of regularization without actually using a huge $\lambda$ in practice, we implement a trick on the critic term in the optimization objective (\cref{alg-line:prac-line1} of \cref{alg:practical}) by smoothing the log density ratio with the log-sigmoid function, which changes the policy update to: 
\begin{align}
    \pi_{t+1} &\leftarrow \arg\min_{\pi\in\Pi} -\frac{1}{n}\sum_{j=1}^{n}\log\sigma\left(\frac{\eta}{2}\log\frac{\pi(y_j^\w|x_j)}{\pi_t(y_j^\w|x_j)} + \frac{\eta}{2}\log\frac{\pi(y_j^\lo|x_j)}{\pi_t(y_j^\lo|x_j)} - \eta\log\frac{\pi(y_j|x_j)}{\pi_t(y_j|x_j)}\right) \notag \\ & \qquad\qquad\qquad\qquad 
    - \frac{\lambda}{n}\sum_{i=1}^n \log\sigma\left(\eta\log\frac{\pi(y_i^\w|x_i)}{\pi_t(y_i^\w|x_i)} - \eta\log\frac{\pi(y_i^\lo|x_i)}{\pi_t(y_i^\lo|x_i)}\right).
\end{align}
This technique, which can also be seen used in \citep{chen2024self,wang2024transforming} for numerical stability purposes, would keep the critic term from being too negative and balance the strength of the critic with that of the regularization, allowing us to use $\lambda$ as small as $1.0$ in the experiments. 

\paragraph{Evaluation.} We evaluate models finetuned with our method and baseline methods on Huggingface Open LLM Leaderboard \citep{open-llm-leaderboard}, which is a commonly used benchmark to test the reasoning ability of a model from multiple angles. Specifically, it includes six NLP tasks testing the model over multitask, commonsense reasoning and inference, questions in various subjects such as science and math questions. Further details about the evaluation process can be found in \cref{sec:appendix-exp}.

\begin{table}[t]
\centering
\begin{adjustbox}{max width=\textwidth}
\begin{tabular}{c | c c c c c c | c}
Model & ARC & HellaSwag & MMLU & TruthfulQA & Winogrande & GSM8k & Average\\
\hline
{\sf mistral-7b-sft-$\beta$} & 58.11 & 82.15 & 59.78 & 43.03 & 77.90 & 39.65 & 60.10\\
\hline
{\sf zephyr-7b-$\beta$} & 63.40 & 84.34 & 59.77 & 55.14 & 77.58 & 33.51 & 62.29\\
{\sf DPO} & 62.20 & 84.65 & 60.22 & 49.52 & 79.08 & 36.01 & 61.95\\
\hline
{\sf SPIN}-iter0 (ours)& 59.98 & 82.48 & 58.87 & \textbf{57.66} & 78.14 & 38.97 & 62.68\\
{\sf SPIN}-iter1 (ours)& 58.87 & 81.58 & 59.33 & 54.49 & 78.45 & 41.24 & 62.34\\
{\sf SPIN}-iter2 (ours)& 57.59 & 81.72 & 57.75 & 53.97 & 77.51 & 34.87 & 60.57\\
\hline
{\sf SPAC}-iter0 & 62.20 & 83.57 & \textbf{60.91} & 54.43 & 78.53 & 38.67 & 63.05\\
{\sf SPAC}-iter1 & 64.33 & 84.46 & 60.13 & 55.54 & \textbf{80.11} & \textbf{41.55} & 64.35\\
{\sf SPAC}-iter2 & \textbf{65.10} & \textbf{85.09} & 60.18 & 55.78 & 79.79 & 40.86 & \textbf{64.47}
\end{tabular}
\end{adjustbox}
\caption{Evaluations of {\sf SPAC} and baselines based on {\sf mistral-7b-sft-beta} using Open LLM Leaderboard.}\label{fig:openllm}
\end{table}

\subsection{Results and Analysis}

We evaluate the models from every iteration of our \alg\ finetuning on the tasks from Open LLM Leaderboard. Despite a lack of directly comparable methods, we also finetune with state-of-the-art methods that do not have guarantees of finding the optimal policy, namely {\sf DPO} and {\sf SPIN} on the same preference dataset and evaluate the resulting models. In particular, since {\sf SPIN} is also an iterative algorithm, we run it for the same number of iterations as \alg. Furthermore, we treat the chosen responses in UltraFeedback as demonstrations when running {\sf SPIN}, for {\sf SPIN} takes a supervised dataset made of prompt-response pairs as input. While this is slightly different from the setting of its original experiments in \citet{chen2024self}, the responses in UltraFeedback are generated from LLMs much stronger than our base model and suffice to serve as demonstrations. We use the training hyperparameters from \citet{chen2024self} and finetune {\sf SPAC} with them too. In addition, we compare with {\sf DPO} by finetuning the base model with it for 2 epochs on the entire UltraFeedback dataset. In addition, we evaluate {\sf zephyr-7b-beta} \citep{tunstall2023zephyr}, which is an industry-grade LLM finetuned from {\sf mistral-7b-sft-beta} using {\sf DPO} on the UltraFeedback dataset. We use 42 as random seed in all experiments. The results are presented in \cref{fig:openllm}.

In our experiments, {\sf SPAC} improves from the initial SFT model by 2.95\% on average across the tasks from Open LLM Leaderboard after the first iteration. Specifically, {\sf SPAC} performs better than or on par with {\sf DPO} in the first 5 tasks and significantly better on GSM8k. In fact, both {\sf SPAC} and {\sf SPIN} are immediately better than DPO after the first iteration, despite only using 1/3 of the training data, but {\sf DPO} noticeably outperforms {\sf SPIN} on ARC and HellaSwag. Over the following two iterations, the performance of {\sf SPAC} continues to increase over all tasks, except for a decrease by 0.78\% on MMLU at the second iteration and a decrease by 0.69\% on GSM8k at the third iteration. In fact, a performance decrease on GSM8k is also witnessed in {\sf DPO}. A reasonable explanation for this is our training data UltraFeedback does not contain mathematical reasoning problems, which GSM8k specifically tests, so finetuning solely on UltraFeedback can cause the LLM to forget its mathematical reasoning ability from pretraining and SFT. In contrast, the performance of {\sf SPIN} starts to waver or decrease after the first iteration over all six tasks. We conjecture the use of pessimism controls the model's overfitting to the training data, so {\sf SPAC} experiences less forgetting than other baselines during finetuning. In the end, we observe the model improvement becomes incremental and the loss near-zero at the third iteration, implying {\sf SPAC} is close to convergence under this set of hyperparameters; {\sf SPAC} improves the base SFT model by a total of 4.37\% on average, more than other baselines by a considerable margin.

\section{Conclusion}

In this work, we propose a new offline preference optimization method named \alg, which is the first that provably converges to the optimal policy under single-policy concentrability and has a scalable implementation for general function approximation at the same time. It is inspired from a Stackelberg game formulation of the offline RL problem and can be implemented as a single-timescale iterative self-play algorithm owing to a change-of-variable trick. We provide a theoretical analysis showing the suboptimality of \alg decays at the rate of $\widetilde{O}(\sqrt{1/n}+\sqrt{1/T})$, ignoring logarithmic and constant factors, under single-policy concentrability. In addition, we finetune a 7B LLM with \alg and evaluate the finetuned model on Open LLM Leaderboard. We demonstrate that \alg attains competitive empirical performance compared to state-of-the-art alignment methods. 

\bibliographystyle{apalike}
\bibliography{ref}



\newpage


\clearpage
\appendix

\begin{center}
{\LARGE Appendix}
\end{center}

\section{Experiment Details}\label{sec:appendix-exp}

\paragraph{Finetuning.} We used four Nvidia A100s 80GB for finetuning, with a global batch size of 32, max sequence length of 1024 and bfloat16 precision. We use the RMSProp \citep{hinton2012neural} optimizer with no weight decay, linear learning rate scheduler with peak learning rate $5e$-$7$ and $5\%$ warmup. We set $\beta = 0.1$ and $\lambda = 1.0$ for all three iterations. For generation, we use the prompting template ``\#\#\# Instruction: \{prompt\}\textbackslash n\textbackslash n\#\#\# Response: '', which is common in works including \citet{alpaca}.

\paragraph{Evaluation.} For Open LLM Leaderboard evaluation, we use the \hyperlink{https://github.com/EleutherAI/lm-evaluation-harness}{lm-evaluation-harness} repository at v0.4.2. Specifically, it contains six tasks: Arc \citep{clark2018think}, for which we take 25-shot examples and use {\sf acc\_norm} as metric; HellaSwag \citep{zellers2019hellaswag}, for which we take 10-shot examples and use {\sf acc\_norm} as metric; MMLU \citep{hendrycks2021measuring}, for which we take 5-shot examples and use {\sf acc} as metric; TruthfulQA \citep{lin2022truthfulqa}, for which we take zero-shot examples and use {\sf mc2} as metric; Winogrande \citep{DBLP:journals/corr/abs-1907-10641}, for which we take 5-shot examples and use {\sf acc} as metric; GSM8k \citep{DBLP:journals/corr/abs-2110-14168}, for which we take 5-shot examples and use {\sf acc} as metric.

\section{Additional Notation}

For any two probability densities $p_1(\cdot)$ and $p_2(\cdot)$, the total variation distance between $P_1$ and $P_2$, the probability distributions induced by $p_1$ and $p_2$ respectively, is
\begin{equation}
    D_{\mathrm{TV}}(P_1, P_2) = \frac{1}{2}\int_z |p_1(z) - p_2(z)| \mathrm{d} z,
\end{equation}
and the Kullback–Leibler (KL) divergence between $P_1$ and $P_2$, the probability distributions induced by $p_1$ and $p_2$ respectively, is
\begin{equation}
    D_{\mathrm{KL}}(P_1 ~||~ P_2) = -\int_z p_1(z)\log\frac{p_1(z)}{p_2(z)} \mathrm{d} z.
\end{equation}

Let $f:\Omega \to \RR$ be a continuously-differentiable, strictly convex function defined on a convex set $\Omega$. The Bregman divergence between $z,z'\in\Omega$ with respect to $f$ is defined as 
\begin{equation}
    B_{f}(z,z') := f(z) - f(z') - \langle\nabla f(z'), z - z'\rangle.
\end{equation}

For a vector $v$ and a positive semidefinite matrix $A$, $\norm{v}_{A}$ denotes a semi-norm of the vector $v$ with respect to the matrix $A$ with $\norm{v}_{A} = \sqrt{v^\top A v}$.

Similar to the definition in \cref{eq:BTL}, for any $f\in\cF$, we define
\begin{equation}
    P_f(y^\w \succ y^\lo ~|~ x) := \frac{\exp(f(x,y^\w))}{\exp(f(x,y^\w)) + \exp(f(x,y^\lo))}.
\end{equation}

Lastly, if there exists a universal constant $C > 0$ such that $f(\cX) \le Cg(\cX)$ for any instantiation of $\cX$, we can denote this with the notation $f(\cX) \lesssim g(\cX)$. In addition, $g(\cX) \gtrsim f(\cX)$ is defined as an equivalent way of writing $f(\cX) \lesssim g(\cX)$.

\section{Supporting Lemmas}

In this section, let us introduce the established results for the concentration of martingale sequences. The first one is Hoeffding's inequality \cite{azuma67,vershynin2018high}, which can be stated as follows:

\begin{theorem}[Azuma-Hoeffding inequality]\label{thm:hoeffding}
	Consider a filtration $\cF_0\subset \cF_1 \subset \cF_2 \subset \cdots$,
	and let $\mathbb{E}_{i}$ stand for the expectation conditioned
	on $\mathcal{F}_i$. Suppose that $Y_{n}=\sum_{i=1}^{n}X_{i}\in\mathbb{R}$,
	where $\{X_{i}\}$ is a real-valued scalar sequence obeying
    $$\left|X_{i}\right|\leq R\qquad\text{and}\qquad\mathbb{E}_{i-1} \big[X_{i}\big]=0\quad\quad\quad\text{for all }i\ge 1$$
	for some $R<\infty$. 
	Then with probability at least $1-\delta$,
	\begin{equation}
		\left|Y_{n}\right|\leq \sqrt{R^2n\log\frac{1}{\delta}}.\label{eq:hoeffding}
	\end{equation}
\end{theorem}


\section{Proof of \cref{thm:main}}\label{sec:appendix-proof}

\begin{proof}
    We begin by invoking the definition of the suboptimality. Using the definition of $\widehat{\pi}$, we have
    \begin{align}
        \subopt(\widehat{\pi}) &= \frac{1}{T}\sum_{t=1}^{T}\EE_{x\sim\rho}\Big[\big\langle\pi^\star(\cdot ~|~ x) - \pi_{t}(\cdot ~|~ x), r^\star(x,\cdot)\big\rangle\Big] \notag\\
        &= \frac{1}{T}\sum_{t=1}^{T}\Big(\EE_{x\sim\rho, y\sim\pi^\star(\cdot|x)}[r^\star(x,y)] - \EE_{x\sim\rho, y\sim\pi_t(\cdot|x)}[r^\star(x,y)]\Big)\notag\\
        &= \frac{1}{T}\sum_{t=1}^{T}\subopt(\pi_t).
    \end{align}
    Thus, it is sufficient for us to focus on finding an upper bound for each $\subopt(\pi_t)$ for $t=1,\cdots,T$, which can be decomposed into the following few terms:
    \begin{align}
        \subopt(\pi_t) &= \EE_{x\sim\rho, y\sim\pi^\star(\cdot|x)}[r^\star(x,y)] - \EE_{x\sim\rho, y\sim\pi_t(\cdot|x)}[r^\star(x,y)]\notag\\
        &= \EE_{x\sim\rho, y\sim\pi^\star(\cdot|x)}\left[r^\star(x,y) - f_{t}(x,y)\right] + \EE_{x\sim\rho, y\sim\pi^\star(\cdot|x)}[f_{t}(x,y)] - \EE_{x\sim\rho, y\sim\pi_t(\cdot|x)}[f_{t}(x,y)]\notag\\ &\quad + \EE_{x\sim\rho, y\sim\pi_t(\cdot|x)}\left[f_{t}(x,y) - r^\star(x,y)\right]\notag\\
        &= \EE_{x\sim\rho, y\sim\pi^\star(\cdot|x)}\left[r^\star(x,y) - f_{t}(x,y)\right] + \EE_{x\sim\rho, y\sim\pi^\star(\cdot|x)}[f_{t}(x,y)] - \EE_{x\sim\rho, y\sim\pi_t(\cdot|x)}[f_{t}(x,y)]\notag\\ &\quad - \EE_{x\sim\rho, y\sim\pi_{\rref}(\cdot|x)}\left[r^\star(x,y) - f_{t}(x,y)\right]\notag\\ &\quad + \EE_{x\sim\rho, y\sim\pi_t(\cdot|x)}\left[f_{t}(x,y)\right] - \EE_{x\sim\rho, y\sim\pi_{\rref}(\cdot|x)}[f_{t}(x,y)]\notag\\ &\quad - \left(\EE_{x\sim\rho, y\sim\pi_t(\cdot|x)}\left[r^\star(x,y)\right] - \EE_{x\sim\rho, y\sim\pi_{\rref}(\cdot|x)}[r^\star(x,y)]\right)\notag\\
        &= \underbrace{\EE_{x\sim\rho, y\sim\pi^\star(\cdot|x)}\left[r^\star(x,y) - f_{t}(x,y)\right] - \EE_{x\sim\rho, y\sim\pi_{\rref}(\cdot|x)}\left[r^\star(x,y) - f_{t}(x,y)\right]}_{I_{t,1}}\notag\\ &\quad + \underbrace{\EE_{x\sim\rho, y\sim\pi^\star(\cdot|x)}[f_{t}(x,y)] - \EE_{x\sim\rho, y\sim\pi_t(\cdot|x)}[f_{t}(x,y)]}_{I_{t,2}}\notag\\ &\quad + \underbrace{\EE_{x\sim\rho, y\sim\pi_t(\cdot|x)}\left[f_{t}(x,y)\right] - \EE_{x\sim\rho, y\sim\pi_{\rref}(\cdot|x)}[f_{t}(x,y)]}_{I_{t,3}}\notag\\ &\quad - \underbrace{\left(\EE_{x\sim\rho, y\sim\pi_t(\cdot|x)}\left[r^\star(x,y)\right] - \EE_{x\sim\rho, y\sim\pi_{\rref}(\cdot|x)}[r^\star(x,y)]\right)}_{I_{t,4}}. \label{eq:decomp}
    \end{align}
    
    Now, we work on an upper bound for $I_{t,1}$ using the coverage assumption in Assumption \ref{assum:coverage}. For notational convenience, let us adopt the temporary shorthand $Z(x,y_1,y_2;r^\star,f_{t}) := r^\star(x,y_1) - r^\star(x,y_2) - f_{t}(x,y_1) + f_{t}(x,y_2)$. This allows us to write
    \begin{align}
        I_{t,1} &= \EE_{x\sim\rho, y_1\sim\pi^\star(\cdot|x),y_2\sim\pi_{\rref}(\cdot|x)}\left[r^\star(x,y_1) - r^\star(x,y_2) - f_{t}(x,y_1) + f_{t}(x,y_2)\right]\notag\\
        &= \sqrt{\EE_{\substack{x\sim\rho\\ y_1\sim\pi_{\rref}(\cdot|x)\\ y_2\sim\pi_{\rref}(\cdot|x)}}\left[Z^2(x,y_1,y_2;r^\star,f_{t})\right]\frac{\EE_{\substack{x\sim\rho\\ y_1\sim\pi^\star(\cdot|x)\\ y_2\sim\pi_{\rref}(\cdot|x)}}\left[Z(x,y_1,y_2;r^\star,f_{t})\right]^2}{\EE_{\substack{x\sim\rho\\ y_1\sim\pi_{\rref}(\cdot|x)\\ y_2\sim\pi_{\rref}(\cdot|x)}}\left[Z^2(x,y_1,y_2;r^\star,f_{t})\right]}}\notag\\
        &\le \sqrt{\EE_{\substack{x\sim\rho\\ y_1\sim\pi_{\rref}(\cdot|x)\\ y_2\sim\pi_{\rref}(\cdot|x)}}\left[Z^2(x,y_1,y_2;r^\star,f_{t})\right]\sup_{f\in\cF}\frac{\EE_{\substack{x\sim\rho\\ y_1\sim\pi^\star(\cdot|x)\\ y_2\sim\pi_{\rref}(\cdot|x)}}\left[Z(x,y_1,y_2;r^\star,f)\right]^2}{\EE_{\substack{x\sim\rho\\ y_1\sim\pi_{\rref}(\cdot|x)\\ y_2\sim\pi_{\rref}(\cdot|x)}}\left[Z^2(x,y_1,y_2;r^\star,f)\right]}}\notag\\
        &\le \sqrt{\Cfrak^2\EE_{x\sim\rho, y_1\sim\pi_{\rref}(\cdot|x), y_2\sim\pi_{\rref}(\cdot|x)}\left[Z^2(x,y_1,y_2;r^\star,f_{t})\right]}\notag\\
        &\le \sqrt{\frac{\Cfrak^4}{2\lambda^2} + \frac{\lambda^2}{2}\EE_{x\sim\rho, y_1\sim\pi_{\rref}(\cdot|x), y_2\sim\pi_{\rref}(\cdot|x)}\left[Z^2(x,y_1,y_2;r^\star,f_{t})\right]^2}\notag\\
        &\le \frac{\Cfrak^2}{\lambda} + \lambda\EE_{x\sim\rho, y_1\sim\pi_{\rref}(\cdot|x),y_2\sim\pi_{\rref}(\cdot|x)}\left[\big(r^\star(x,y_1) - r^\star(x,y_2) - f_{t}(x,y_1) + f_{t}(x,y_2)\big)^2\right],
    \end{align}
    in which the penultimate step is by $\frac{a^2}{2} + \frac{b^2}{2} \ge ab$ with $a = \Cfrak^2/\lambda$ and $b = \lambda\EE_{x\sim\rho, y_1\sim\pi_{\rref}(\cdot|x), y_2\sim\pi_{\rref}(\cdot|x)}\left[Z^2(x,y_1,y_2;r^\star,f_{t})\right]$, and the last step is by the triangle inequality.

    Combined with $I_{t,3}$, this upper bound on $I_{t,1}$ above leads us to the following upper bound: 
    \begin{align}
        &\quad\ I_{t,1} + I_{t,3} \notag\\
        &\le \frac{\Cfrak^2}{\lambda} + \lambda\EE_{x\sim\rho, y_1\sim\pi_{\rref}(\cdot|x),y_2\sim\pi_{\rref}(\cdot|x)}\left[\big(r^\star(x,y_1) - r^\star(x,y_2) - f_{t}(x,y_1) + f_{t}(x,y_2)\big)^2\right]\notag\\ &\quad + \EE_{x\sim\rho, y\sim\pi_t(\cdot|x)}\left[f_{t}(x,y)\right] - \EE_{x\sim\rho, y\sim\pi_{\rref}(\cdot|x)}[f_{t}(x,y)] \label{eq:I1+I3-1}\\
        &\overset{(\mathrm{i})}{\lesssim} \frac{\Cfrak^2}{\lambda} + \lambda\cE_\cD(f_t) - \lambda\cE_\cD(r^\star) + \lambda\frac{\kappa^2}{n}\log\frac{|\cF|}{\delta}\notag\\ &\quad + \EE_{x\sim\rho, y\sim\pi_t(\cdot|x)}\left[f_{t}(x,y)\right] - \EE_{x\sim\rho, y\sim\pi_{\rref}(\cdot|x)}[f_{t}(x,y)] \label{eq:I1+I3-2}\\
        &= \frac{\Cfrak^2}{\lambda} + \cL_\cD(f_t) + \lambda\cE_\cD(f_t) - \cL_\cD(r^\star) - \lambda\cE_\cD(r^\star) + \lambda\frac{\kappa^2}{n}\log\frac{|\cF|}{\delta}\notag\\ &\quad + \EE_{x\sim\rho, y\sim\pi_t(\cdot|x)}\left[f_{t}(x,y)\right] - \EE_{x\sim\rho, y\sim\pi_{\rref}(\cdot|x)}[f_{t}(x,y)] - \cL_\cD(f_t)\notag\\ &\quad + \cL_\cD(r^\star) - \Big(\EE_{x\sim\rho, y\sim\pi_t(\cdot|x)}\left[r^\star(x,y)\right] - \EE_{x\sim\rho, y\sim\pi_{\rref}(\cdot|x)}[r^\star(x,y)]\Big)
        \notag\\ &\quad + \EE_{x\sim\rho, y\sim\pi_t(\cdot|x)}\left[r^\star(x,y)\right] - \EE_{x\sim\rho, y\sim\pi_{\rref}(\cdot|x)}[r^\star(x,y)]
        \label{eq:I1+I3-3}\\
        &\overset{(\mathrm{ii})}{\le} \frac{\Cfrak^2}{\lambda} + \lambda\frac{\kappa^2}{n}\log\frac{|\cF|}{\delta} + \EE_{x\sim\rho, y\sim\pi_t(\cdot|x)}\left[f_{t}(x,y)\right] - \EE_{x\sim\rho, y\sim\pi_{\rref}(\cdot|x)}[f_{t}(x,y)] - \cL_\cD(f_t)\notag\\ &\quad + \cL_\cD(r^\star) - \Big(\EE_{x\sim\rho, y\sim\pi_t(\cdot|x)}\left[r^\star(x,y)\right] - \EE_{x\sim\rho, y\sim\pi_{\rref}(\cdot|x)}[r^\star(x,y)]\Big)
        \notag\\ &\quad + \EE_{x\sim\rho, y\sim\pi_t(\cdot|x)}\left[r^\star(x,y)\right] - \EE_{x\sim\rho, y\sim\pi_{\rref}(\cdot|x)}[r^\star(x,y)]
        \notag\\
        &\overset{(\mathrm{iii})}{\lesssim} \frac{\Cfrak^2}{\lambda} + \lambda\frac{\kappa^2}{n}\log\frac{|\cF|}{\delta} + \EE_{x\sim\rho, y\sim\pi_t(\cdot|x)}\left[f_{t}(x,y)\right] - \EE_{x\sim\rho, y\sim\pi_{\rref}(\cdot|x)}[f_{t}(x,y)] - \cL_\cD(f_t)\notag\\ &\quad + R\sqrt{\frac{1}{n}\log\frac{1}{\delta}} + \EE_{x\sim\rho, y\sim\pi_t(\cdot|x)}\left[r^\star(x,y)\right] - \EE_{x\sim\rho, y\sim\pi_{\rref}(\cdot|x)}[r^\star(x,y)]
        \notag\\
        &\overset{(\mathrm{iv})}{\lesssim} \frac{\Cfrak^2}{\lambda} + \lambda\frac{\kappa^2}{n}\log\frac{|\cF|}{\delta} + R\sqrt{\frac{1}{n}\log\frac{|\cF|}{\delta}} + R\sqrt{\frac{1}{n}\log\frac{1}{\delta}}\notag\\ &\quad + \EE_{x\sim\rho, y\sim\pi_t(\cdot|x)}\left[r^\star(x,y)\right] - \EE_{x\sim\rho, y\sim\pi_{\rref}(\cdot|x)}[r^\star(x,y)]\notag\\
        &\lesssim \frac{\Cfrak^2}{\lambda} + \lambda\frac{\kappa^2}{n}\log\frac{|\cF|}{\delta} + R\sqrt{\frac{1}{n}\log\frac{|\cF|}{\delta}} + \EE_{x\sim\rho, y\sim\pi_t(\cdot|x)}\left[r^\star(x,y)\right] - \EE_{x\sim\rho, y\sim\pi_{\rref}(\cdot|x)}[r^\star(x,y)]. \label{eq:I1+I3-f}
    \end{align}
    Here, (i) is obtained by invoking \cref{lem:MLE} on the second term in \cref{eq:I1+I3-1}. 
    
    (ii) is because the second and third term of \cref{eq:I1+I3-3} coincide with the minimization objective in \cref{alg-line:theory-line1} of \cref{alg:theoretical}. By construction, $\cL_\cD(f_{t}, \pi_t) + \lambda \cE_\cD(f_{t}) = \min_{f\in\cF} \cL_\cD(f, \pi_t) + \lambda \cE_\cD(f)$, so we have $\cL_\cD(f_{t}, \pi_t) + \lambda \cE_\cD(f_{t}) \le \cL_\cD(f', \pi_t) + \lambda \cE_\cD(f')$ for any $f'\in\cF$. 
    
    (iii) is obtained by using the following upper bound with the use of Hoeffding's inequality (Theorem \ref{thm:hoeffding}). With probability at least $1-\delta/3$,
    \begin{align}
        &\quad\ \cL_\cD(r^\star) - \Big(\EE_{x\sim\rho, y\sim\pi_t(\cdot|x)}\left[r^\star(x,y)\right] - \EE_{x\sim\rho, y\sim\pi_{\rref}(\cdot|x)}[r^\star(x,y)]\Big) \notag\\
        &= \left(\frac{1}{n}\sum_{j=1}^{n}r^\star(x_j,y_j) - \EE_{x\sim\rho, y\sim\pi_t(\cdot|x)}\left[r^\star(x,y)\right]\right) - \left(\frac{1}{n}\sum_{j=1}^{n}r^\star(x_j,y_j^\w) - \EE_{x\sim\rho, y\sim\pi_{\rref}(\cdot|x)}\left[r^\star(x,y)\right]\right) \notag\\
        &\le \left|\frac{1}{n}\sum_{j=1}^{n}r^\star(x_j,y_j) - \EE_{x\sim\rho, y\sim\pi_t(\cdot|x)}\left[r^\star(x,y)\right]\right| + \left|\frac{1}{n}\sum_{j=1}^{n}r^\star(x_j,y_j^\w) - \EE_{x\sim\rho, y\sim\pi_{\rref}(\cdot|x)}\left[r^\star(x,y)\right]\right| \notag\\
        &\lesssim R\sqrt{\frac{1}{n}\log\frac{1}{\delta}}.
    \end{align}

    (iv) is obtained by using the following upper bound with the use of Hoeffding's inequality (Theorem \ref{thm:hoeffding}) on any $f\in\cF$, combined with the use of the union bound over all $f\in\cF$. With probability at least $1-\delta/3$,
    \begin{align}
        &\quad\ \EE_{x\sim\rho, y\sim\pi_t(\cdot|x)}\left[f_{t}(x,y)\right] - \EE_{x\sim\rho, y\sim\pi_{\rref}(\cdot|x)}[f_{t}(x,y)] - \cL_\cD(f_t)\notag\\
        &= \left(\frac{1}{n}\sum_{j=1}^{n}f_{t}(x_j,y_j^\w) - \EE_{x\sim\rho, y\sim\pi_{\rref}(\cdot|x)}\left[f_{t}(x,y)\right]\right) - \left(\frac{1}{n}\sum_{j=1}^{n}f_{t}(x_j,y_j) - \EE_{x\sim\rho, y\sim\pi_t(\cdot|x)}\left[f_{t}(x,y)\right]\right) \notag\\
        &\le \left|\frac{1}{n}\sum_{j=1}^{n}f_{t}(x_j,y_j^\w) - \EE_{x\sim\rho, y\sim\pi_{\rref}(\cdot|x)}\left[f_{t}(x,y)\right]\right| - \left|\frac{1}{n}\sum_{j=1}^{n}f_{t}(x_j,y_j) - \EE_{x\sim\rho, y\sim\pi_t(\cdot|x)}\left[f_{t}(x,y)\right]\right| \notag\\
        &\lesssim R\sqrt{\frac{1}{n}\log\frac{|\cF|}{\delta}}.
    \end{align}

    Bringing \cref{eq:I1+I3-f} back into \cref{eq:decomp}, we have 
    \begin{align}
        \subopt(\pi_t) &\lesssim \EE_{x\sim\rho, y\sim\pi^\star(\cdot|x)}[f_{t}(x,y)] - \EE_{x\sim\rho, y\sim\pi_t(\cdot|x)}[f_{t}(x,y)]\notag\\ & \quad + \EE_{x\sim\rho, y\sim\pi_t(\cdot|x)}\left[r^\star(x,y)\right] - \EE_{x\sim\rho, y\sim\pi_{\rref}(\cdot|x)}[r^\star(x,y)]\notag\\ & \quad - \EE_{x\sim\rho, y\sim\pi_t(\cdot|x)}\left[r^\star(x,y)\right] - \EE_{x\sim\rho, y\sim\pi_{\rref}(\cdot|x)}[r^\star(x,y)] \notag\\ \ & \quad + \frac{\Cfrak^2}{\lambda} + R\sqrt{\frac{1}{n}\log\frac{|\cF|}{\delta}} + \lambda\frac{\kappa^2}{n}\log\frac{|\cF|}{\delta}\notag\\
        &= \EE_{x\sim\rho, y\sim\pi^\star(\cdot|x)}[f_{t}(x,y)] - \EE_{x\sim\rho, y\sim\pi_t(\cdot|x)}[f_{t}(x,y)] \notag\\ \ & \quad + \frac{\Cfrak^2}{\lambda} + R\sqrt{\frac{1}{n}\log\frac{|\cF|}{\delta}} + \lambda\frac{\kappa^2}{n}\log\frac{|\cF|}{\delta}. \label{eq:decomp-1}
    \end{align}
    Hence, we only need to focus on finding an upper bound on the sum of $I_{t,2}$'s for the rest of this proof.
    \begin{align}
        \sum_{t=1}^{T}I_{t,2} &= \sum_{t=1}^{T}\EE_{x\sim\rho, y\sim\pi^\star(\cdot|x)}[f_{t}(x,y)] - \EE_{x\sim\rho, y\sim\pi_t(\cdot|x)}[f_{t}(x,y)] \notag\\
        &= \sum_{t=1}^{T}\EE_{x\sim\rho}\Big[\big\langle\pi^\star(\cdot ~|~ x) - \pi_{t}(\cdot ~|~ x), f_t(x,\cdot)\big\rangle \Big]\notag\\
        &\le \frac{2R^2T}{\eta} + \eta\EE_{x\sim\rho}\Big[D_{\mathrm{KL}}\big(\pi^\star(\cdot~|~x)~||~\pi_{\init}(\cdot~|~x)\big)\Big] \notag\\
        &\lesssim R\sqrt{\EE_{x\sim\rho}\Big[D_{\mathrm{KL}}\big(\pi^\star(\cdot~|~x)~||~\pi_{\init}(\cdot~|~x)\big)\Big]\cdot T}, \label{eq:I_2}
    \end{align}
    in which the penultimate step is by \cref{lem:no-regret-3} and the last step is by plugging in the choice of $\eta$.   
    
    Finally, bringing \cref{eq:I_2} into \cref{eq:decomp-1} and plugging in the choice of $\lambda$, we have
    \begin{align}
        \subopt(\widehat{\pi}) &= \frac{1}{T}\sqrt{\EE_{x\sim\rho}\Big[D_{\mathrm{KL}}\big(\pi^\star(\cdot~|~x)~||~\pi_{\init}(\cdot~|~x)\big)\Big]\cdot T} + \frac{\Cfrak^2}{\lambda} + R\sqrt{\frac{1}{n}\log\frac{|\cF|}{\delta}} + \lambda\frac{\kappa^2}{n}\log\frac{|\cF|}{\delta} \notag\\
        &\lesssim R\sqrt{\frac{\EE_{x\sim\rho}\Big[D_{\mathrm{KL}}\big(\pi^\star(\cdot~|~x)~||~\pi_{\init}(\cdot~|~x)\big)\Big]}{T}} + \Cfrak\kappa\sqrt{\frac{1}{n}\log\frac{|\cF|}{\delta}} + R\sqrt{\frac{1}{n}\log\frac{|\cF|}{\delta}} .
    \end{align}
	Since $\log|\cF|$ and $\log|\Pi|$ only differ by a constant factor and are on the same order, we can arrive at the advertised bound.
\end{proof}

\subsection{Proof of MLE (\cref{lem:MLE})}

In this subsection, we present the additional steps to obtain an upper bound on $I_{t,1}$, which is controlled by the MLE regularization in the objective of our algorithm. Similar results can be found in \citet{bhardwaj2024adversarial} and \citet{zhan2024provable}, and our lemmas are an adaptation of theirs to our setting.

\begin{lemma}\label{lem:MLE}
    Assume all functions $f \in \cF$ are uniformly bounded by a constant $R > 0$. For any $0< \delta < 1$, with probability at least $1-\delta$,
    \begin{equation*}
        \EE_{x\sim\rho, y_1\sim\pi_{\rref}(\cdot|x),y_2\sim\pi_{\rref}(\cdot|x)}\left[\big(r^\star(x,y_1) - r^\star(x,y_2) - f(x,y_1) + f(x,y_2)\big)^2\right] \lesssim \frac{\kappa^2}{n}\log \frac{|\cF|}{\delta}
    \end{equation*}
    for all $f\in\cF$ simultaneously.
\end{lemma}

\begin{proof}
    By plugging the conclusion of \cref{lem:MLE-bracket} into \cref{lem:MLE-TV}, we have
    \begin{align}
        &\quad\
        \EE_{\substack{x\sim\rho\\ y_1\sim\pi_{\rref}(\cdot|x)\\ y_2\sim\pi_{\rref}(\cdot|x)}}\Big[D_{\mathrm{TV}}\big(P_f(\cdot ~|~ x,y_1,y_2), P(\cdot ~|~ x,y_1,y_2)\big)^2\Big] \notag\\
        &\lesssim \frac{1}{n}\sum_{i=1}^n\log\frac{ P(y_i^\w \succ y_i^\lo ~|~ x_i)}{P_f(y_i^\w \succ y_i^\lo ~|~ x_i)} + \frac{1}{n}\log\frac{|\cF|}{\delta}\notag\\
        &\lesssim \frac{1}{n}\log\frac{|\cF|}{\delta}.
    \end{align}

     Finally, recall the definitions
    \begin{equation*}
        P(y^\w \succ y^\lo ~|~ x) = \sigma(r^\star(x,y^\w) - r^\star(x,y^\lo))
    \end{equation*}
    and for any $f\in\cF$,
    \begin{equation*}
        P_f(y^\w \succ y^\lo ~|~ x) = \sigma(f(x,y^\w) - f(x,y^\lo)).
    \end{equation*}
    With these, we can apply the Mean Value Theorem on the sigmoid function between $r^\star(x,y_1) - r^\star(x,y_2)$ and $f(x,y_1) - f(x,y_2)$ to arrive at the advertised bound with an additional factor of $\kappa^2 = 1/(\inf_{z\in[-R,R]} \sigma'(z))^2$. Specifically, we have
    \begin{align*}
        &\quad\ \EE_{x\sim\rho, y_1\sim\pi_{\rref}(\cdot|x),y_2\sim\pi_{\rref}(\cdot|x)}\left[\big(r^\star(x,y_1) - r^\star(x,y_2) - f(x,y_1) + f(x,y_2)\big)^2\right] \notag\\
        &\lesssim \kappa^2\EE_{\substack{x\sim\rho\\ y_1\sim\pi_{\rref}(\cdot|x)\\ y_2\sim\pi_{\rref}(\cdot|x)}}\bigg[\Big(\left|P_f(y_1\succ y_2 ~|~ x) - P(y_1 \succ y_2 ~|~ x)\right| + \left|P_f(y_1 \prec y_2 ~|~ x) - P(y_1 \prec y_2 ~|~ x)\right|\Big)^2\bigg]\\
        &= 2\kappa^2\EE_{\substack{x\sim\rho\\ y_1\sim\pi_{\rref}(\cdot|x)\\ y_2\sim\pi_{\rref}(\cdot|x)}}\Big[D_{\mathrm{TV}}\big(P_f(\cdot ~|~ x,y_1,y_2), P(\cdot ~|~ x,y_1,y_2)\big)^2\Big] \notag\\
        &\lesssim \frac{\kappa^2}{n}\log \frac{|\cF|}{\delta},
    \end{align*}
    which leads us to the advertised bound.
\end{proof}

\subsubsection{\cref{lem:MLE-bracket} and Its Proof}

\begin{lemma}\label{lem:MLE-bracket}
    For any $0< \delta < 1$, with probability at least $1-\delta$,
    \begin{equation}
        \sum_{i=1}^n\log\frac{ P_f(y_i^\w \succ y_i^\lo ~|~ x_i)}{P(y_i^\w \succ y_i^\lo ~|~ x_i)} \le \log \frac{|\cF|}{\delta} 
    \end{equation}
    for all $f\in\cF$ simultaneously.
\end{lemma}

\begin{proof}
    This proof follows a standard argument of MLE \citep{Geer2000EmpiricalPI}. For any $g\in\cF$,
    \begin{align}
        &\quad\ \EE\left[\exp\left(\sum_{i=1}^n\log\frac{ P_g(y_i^\w \succ y_i^\lo ~|~ x_i)}{P(y_i^\w \succ y_i^\lo ~|~ x_i)}\right)\right]\notag\\
        &\overset{(\mathrm{i})}{=} \prod_{i=1}^n\EE\left[\exp\left(\log\frac{ P_g(y_i^\w \succ y_i^\lo ~|~ x_i)}{P(y_i^\w \succ y_i^\lo ~|~ x_i)}\right)\right]\notag\\
        &= \prod_{i=1}^n\EE\left[\frac{ P_g(y_i^\w \succ y_i^\lo ~|~ x_i)}{P(y_i^\w \succ y_i^\lo ~|~ x_i)}\right]\notag\\
        &\overset{(\mathrm{ii})}{=}  \prod_{i=1}^n\EE\Bigg[\EE\bigg[\frac{ P_g(y_i^\w \succ y_i^\lo ~|~ x_i)}{P(y_i^\w \succ y_i^\lo ~|~ x_i)}~|~ (x_i,y_i^\w,y_i^\lo)\bigg]\Bigg]\notag\\
        &= \prod_{i=1}^n\EE\left[ P_g(y_i^\w \succ y_i^\lo ~|~ x_i) + P_g(y_i^\w \prec y_i^\lo ~|~ x_i)\right] = 1.
    \end{align}
    Above, (i) the first step is by the independence of the samples in $\cD$. (ii) is by the Tower property.

    By Markov's inequality, we can have
    \begin{align}
        &\quad\ \PP\left[\sum_{i=1}^n\log\frac{ P_g(y_i^\w \succ y_i^\lo ~|~ x_i)}{P(y_i^\w \succ y_i^\lo ~|~ x_i)} > \log\frac{1}{\delta}\right] \notag\\
        &\le \EE\left[\exp\left(\sum_{i=1}^n\log\frac{ P_g(y_i^\w \succ y_i^\lo ~|~ x_i)}{P(y_i^\w \succ y_i^\lo ~|~ x_i)}\right)\right]\exp\left(-\log\frac{1}{\delta}\right) \le \delta.
    \end{align}

    Furthermore, we can invoke the union bound over all $g\in \cF$ and have
    \begin{align}
        \PP\left[\sum_{i=1}^n\log\frac{ P_g(y_i^\w \succ y_i^\lo ~|~ x_i)}{P(y_i^\w \succ y_i^\lo ~|~ x_i)} > \log\frac{|\cF|}{\delta}\right] \le \delta.
    \end{align}
    At this point, we have arrived at the advertised bound.
\end{proof}

\subsubsection{\cref{lem:MLE-TV} and Its Proof}

Note that $\cD$ can be equivalently written in the form of $\cD = \{(x_i,y_i,y_i',o_i)\}_{i=1}^n$, where $o_i$ is a binary symbol that indicates whether $y_i$ is preferred to $y_i'$ given $x_i$, i.e., $o_i = +$ when $y_i \succ y_i' ~|~ x_i$ and $o_i = -$ otherwise. Using this set of notation, let us define $P(+ ~|~ x,y_1,y_2) = P(y_1 \succ y_2 ~|~ x)$ and $P(- ~|~ x,y_1,y_2) = P(y_1 \prec y_2 ~|~ x)$ and similarly, $P_f(+ ~|~ x,y_1,y_2) = P_f(y_1 \succ y_2 ~|~ x)$ and $P_f(- ~|~ x,y_1,y_2) = P_f(y_1 \prec y_2 ~|~ x)$. In the following lemma and its proof, we will make use of this set of notation for convenience.

\begin{lemma}\label{lem:MLE-TV}
     For any $0< \delta < 1$, with probability at least $1-\delta$,
    \begin{align*}
        &\EE_{\substack{x\sim\rho\\ y_1\sim\pi_{\rref}(\cdot|x)\\ y_2\sim\pi_{\rref}(\cdot|x)}}\Big[D_{\mathrm{TV}}\big(P_f(\cdot ~|~ x,y_1,y_2), P(\cdot ~|~ x,y_1,y_2)\big)^2\Big] \lesssim \frac{1}{n}\left(\sum_{i=1}^n\log\frac{ P(y_i^\w \succ y_i^\lo ~|~ x_i)}{P_f(y_i^\w \succ y_i^\lo ~|~ x_i)} + \log\frac{|\cF|}{\delta}\right)
    \end{align*}
    for all $f\in\cF$ simultaneously.
\end{lemma}

\begin{proof}
    By Lemma 25 of \citet{agarwal2020flambe}, we have
    \begin{align}
        &\quad\ \EE_{\substack{x\sim\rho\\ y_1\sim\pi_{\rref}(\cdot|x)\\ y_2\sim\pi_{\rref}(\cdot|x)}}\Big[D_{\mathrm{TV}}\big(P_f(\cdot ~|~ x,y_1,y_2), P(\cdot ~|~ x,y_1,y_2)\big)^2\Big] \notag\\
        &\le -2\log\EE_{\substack{x\sim\rho\\ y\sim\pi_{\rref}(\cdot|x)\\ y'\sim\pi_{\rref}(\cdot|x)\\ o\sim P(y \succ y' | x)}}\left[\exp\left(-\frac{1}{2}\log\frac{ P(o ~|~ x,y,y')}{P_f(o ~|~ x,y,y')}\right)\right]. \label{eq:MLE-TV-1}
    \end{align}

    Now, define $\widetilde{\cD} = \{(\widetilde{x}_i, \widetilde{y}_i, \widetilde{y}_i', \widetilde{o}_i)\}$ be a dataset that is an independent copy of $\cD$ from the same sampling distribution. We have
    \begin{align}
        &\quad\ -\log\EE_{\substack{x\sim\rho\\ y\sim\pi_{\rref}(\cdot|x)\\ y'\sim\pi_{\rref}(\cdot|x)\\ o\sim P(y \succ y' | x)}}\left[\exp\left(-\frac{1}{2}\log\frac{ P(o ~|~ x,y,y')}{P_f(o ~|~ x,y,y')}\right)\right] \notag\\
        &= -\frac{1}{n}\sum_{i=1}^n\log  \EE_{\substack{x\sim\rho\\ y\sim\pi_{\rref}(\cdot|x)\\ y'\sim\pi_{\rref}(\cdot|x)\\ o\sim P(y \succ y' | x)}}\left[\exp\left(-\frac{1}{2}\log\frac{ P(o ~|~ x,y,y')}{P_f(o ~|~ x,y,y')}\right)\right]\notag\\
        &= -\frac{1}{n}\log  \EE_{\widetilde{\cD}}\left[\exp\left(-\frac{1}{2}\sum_{i=1}^n\log\frac{ P(\widetilde{o}_i ~|~ \widetilde{x}_i, \widetilde{y}_i, \widetilde{y}_i')}{P_f(\widetilde{o}_i ~|~ \widetilde{x}_i, \widetilde{y}_i, \widetilde{y}_i')}\right) ~|~ \cD\right],\label{eq:MLE-TV-2}
    \end{align}
    in which the last step is by the independence of the samples in $\cD$ and $\widetilde{\cD}$.

    Let us temporary define $\ell\big(f,(x,y,y',o)\big) := -\frac{1}{2}\log\frac{ P(o ~|~ x,y,y')}{P_f(o ~|~ x,y,y')}$. By Lemma 24 of \citet{agarwal2020flambe}, we have
    \begin{align}
        \EE_{\cD}\left[\exp\left(\sum_{i=1}^n \ell\big(f,(x_i, y_i, y_i', o_i)\big) - \log \EE_{\widetilde{\cD}} \left[\exp\left(\sum_{i=1}^n\ell\big(f,(\widetilde{x}_i, \widetilde{y}_i, \widetilde{y}_i', \widetilde{o}_i)\big)\right) ~|~ \cD\right] - \log|\cF|\right)\right] \le 1. \notag
    \end{align}

    By applying the Chernoff bound and then the union bound for every $f\in\cF$ to the inequality above, we have that with probability at least $1-\delta$,
    \begin{align}
        -\log \EE_{\widetilde{\cD}} \left[\exp\left(\sum_{i=1}^n\ell\big(f,(\widetilde{x}_i, \widetilde{y}_i, \widetilde{y}_i', \widetilde{o}_i)\big)\right) ~|~ \cD\right] \le -\sum_{i=1}^n \ell\big(f,(x_i, y_i, y_i', o_i)\big) + 2\log\frac{|\cF|}{\delta}. \label{eq:MLE-TV-3}
    \end{align}

    Finally, bringing \cref{eq:MLE-TV-3} into \cref{eq:MLE-TV-2} and then back into \cref{eq:MLE-TV-1}, we have that with probability at least $1-\delta$,
    \begin{align}
        &\quad\ \EE_{\substack{x\sim\rho\\ y_1\sim\pi_{\rref}(\cdot|x)\\ y_2\sim\pi_{\rref}(\cdot|x)}}\Big[D_{\mathrm{TV}}\big(P_f(\cdot ~|~ x,y_1,y_2), P(\cdot ~|~ x,y_1,y_2)\big)^2\Big] \notag\\
        &\le -\frac{2}{n}\sum_{i=1}^n \ell\big(f,(x_i, y_i, y_i', o_i)\big) + \frac{4}{n}\log\frac{|\cF|}{\delta} \notag\\
        &= -\frac{1}{n}\sum_{i=1}^n \log\frac{ P(o_i ~|~ x_i, y_i, y_i')}{P_f(o_i ~|~ x_i, y_i, y_i')} + \frac{4}{n}\log\frac{|\cF|}{\delta}\\
        &= -\frac{1}{n}\sum_{i=1}^n \log\frac{ P(y_i^\w \succ y_i^\lo ~|~ x_i)}{P_f(y_i^\w \succ y_i^\lo ~|~ x_i)} + \frac{4}{n}\log\frac{|\cF|}{\delta},
    \end{align}
    which is the advertised bound in this lemma.
\end{proof}

\subsection{Bounding $I_{t,2}$ in \cref{eq:decomp}}

In this subsection, we present the steps to obtain an upper bound on $I_{t,2}$. Note that the sum of $I_{t,2}$'s actually coincides with the definition of regret in online learning \citep{hazan2016introduction}. 
\begin{equation*}
    \sum_{t=1}^{T}I_{t,2} = \sum_{t=1}^{T}\EE_{x\sim\rho, y\sim\pi^\star(\cdot|x)}[f_{t}(x,y)] - \EE_{x\sim\rho, y\sim\pi_t(\cdot|x)}[f_{t}(x,y)].
\end{equation*}
Thus, the size of it depends on the no-regret policy optimization update in \cref{alg-line:theory-line2} of \cref{alg:theoretical}, and the derivation of its upper bound largely follows from the standard analysis of policy mirror descent. Similar results can be found in \citet{xie2021bellmanconsistent} (Theorem C.5), and our lemma is a modification of theirs to accommodate arbitrary initial policy initial policy our setting.

\subsubsection{\cref{lem:no-regret-1} and Its Proof}

Define 
\begin{align}
    \ell_x(\pi) &:= \eta D_{\mathrm{KL}}\big(\pi(\cdot|x) ~||~ \pi_{\rref}(\cdot|x)\big) \notag\\
    &= \eta\sum_{y\in\cY}\pi(y~|~x)\log\pi(y~|~x) - \eta\sum_{y\in\cY}\pi(y~|~x)\log\pi_{\rref}(y~|~x). \notag
\end{align}

\begin{lemma}\label{lem:no-regret-1}
    For any policy $\pi:\cX\to\Delta(\cY)$ and $x\in\cX$,
    \begin{align}
        \sum_{t=1}^{T}\big\langle\pi(\cdot ~|~ x), f_t(x,\cdot)\big\rangle - \ell_x(\pi) \le \sum_{t=1}^{T}\big\langle\pi_{t+1}(\cdot ~|~ x), f_t(x,\cdot)\big\rangle - \ell_x(\pi_1).
    \end{align}
\end{lemma}

\begin{proof}
    We prove this lemma by induction. The base case $T=0$ trivially holds since the KL divergence is always non-negative and is minimized when the two probability distributions are the same. For the inductive step, assuming the inductive hypothesis is true for $T = T'$, we have that for any $\pi$ and $x\in\cX$,
    \begin{align}
        &\quad\ \sum_{t=1}^{T'+1}\big\langle\pi_{t+1}(\cdot ~|~ x), f_t(x,\cdot)\big\rangle - \ell_x(\pi_1) \notag\\
        &\overset{(\mathrm{i})}{\le} \sum_{t=1}^{T'+1}\big\langle\pi_{T'+2}(\cdot ~|~ x), f_t(x,\cdot)\big\rangle - \ell_x(\pi_{T'+2}) \notag\\
        &= \sum_{t=1}^{T'}\big\langle\pi_{T'+2}(\cdot ~|~ x), f_t(x,\cdot)\big\rangle - \ell_x(\pi_{T'+2}) + \big\langle\pi_{T'+2}(\cdot ~|~ x), f_{T'+1}(x,\cdot)\big\rangle \notag\\
        &\overset{(\mathrm{ii})}{\le} \sum_{t=1}^{T'}\big\langle\pi_{t+1}(\cdot ~|~ x), f_t(x,\cdot)\big\rangle - \ell_x(\pi_1) + \big\langle\pi_{T'+2}(\cdot ~|~ x), f_{T'+1}(x,\cdot)\big\rangle \notag\\
        &= \sum_{t=1}^{T}\big\langle\pi(\cdot ~|~ x), f_t(x,\cdot)\rangle - \ell_x(\pi),
    \end{align}
    which is the advertised bound.
    
    In the series of equalities and inequalities above, (ii) is by the inductive hypothesis.
    
    (i) can be obtained because $\pi_{T'+2}$ is the global maximizer of $\sum_{t=1}^{T'+1}\big\langle\pi(\cdot ~|~ x), f_t(x,\cdot)\big\rangle - \ell_x(\pi)$. This can also be proven inductively following the proof for the Follow the Regularized Leader (FTRL) algorithm in online learning, with the analytical solution of the policy from DPO being the base case.
\end{proof}

\subsubsection{\cref{lem:no-regret-2} and Its Proof}

\begin{lemma}\label{lem:no-regret-2}
    For any policy $\pi:\cX\to\Delta(\cY)$ and $x\in\cX$,
    \begin{align}
        \sum_{t=1}^{T}\big\langle\pi(\cdot ~|~ x) - \pi_{t}(\cdot ~|~ x), f_t(x,\cdot)\big\rangle \le \sum_{t=1}^{T}\big\langle\pi_{t+1}(\cdot ~|~ x) - \pi_{t}(\cdot ~|~ x), f_t(x,\cdot)\big\rangle - \ell_x(\pi_1) + \ell_x(\pi). \notag
    \end{align}
\end{lemma}

\begin{proof}
    \begin{align}
        &\quad\ \sum_{t=1}^{T}\big\langle\pi(\cdot ~|~ x) - \pi_{t}(\cdot ~|~ x), f_t(x,\cdot)\big\rangle \notag\\
        &= \sum_{t=1}^{T}\big\langle\pi(\cdot ~|~ x), f_t(x,\cdot)\big\rangle - \ell_x(\pi) + \ell_x(\pi) - \sum_{t=1}^{T}\big\langle\pi_{t}(\cdot ~|~ x), f_t(x,\cdot)\big\rangle \notag\\
        &\le \sum_{t=1}^{T}\big\langle\pi_{t+1}(\cdot ~|~ x), f_t(x,\cdot)\big\rangle - \ell_x(\pi_1) - \sum_{t=1}^{T}\big\langle\pi_{t}(\cdot ~|~ x), f_t(x,\cdot)\big\rangle + \ell_x(\pi) \notag\\
        &= \sum_{t=1}^{T}\big\langle\pi_{t+1}(\cdot ~|~ x) - \pi_{t}(\cdot ~|~ x), f_t(x,\cdot)\big\rangle - \ell_x(\pi_1) + \ell_x(\pi),
    \end{align}
    in which the penultimate step is by \cref{lem:no-regret-1} and the last step is by the non-positivity of $\ell_x(\cdot)$.
\end{proof}

\subsubsection{\cref{lem:no-regret-3} and Its Proof}

\begin{lemma}\label{lem:no-regret-3}
    For any policy $\pi:\cX\to\Delta(\cY)$ and $x\in\cX$,
    \begin{align}
        \sum_{t=1}^{T}\big\langle\pi(\cdot ~|~ x) - \pi_{t}(\cdot ~|~ x), f_t(x,\cdot)\big\rangle \le 2\eta R^2T + \frac{1}{\eta}D_{\mathrm{KL}}\big(\pi(\cdot~|~x)~||~\pi_{\rref}(\cdot~|~x)\big).
    \end{align}
\end{lemma}

\begin{proof}
    First, let us define 
    \begin{equation*}
        L_{x,t}(\pi) := \sum_{t'=1}^{t}\big\langle\pi(\cdot ~|~ x), f_{t'}(x,\cdot)\big\rangle - \ell_x(\pi).
    \end{equation*}
    Denote the Bregman divergence with respect to $L_{x,t}(\cdot)$ with $B_{L_{x,t}}(\cdot||\cdot)$ and similarly denote the Bregman divergence with respect to $\ell_{x}(\cdot)$ with $B_{\ell_{x}}(\cdot||\cdot)$. Notice
    \begin{align}
        L_{x,t}(\pi_t) &= L_{x,t}(\pi_{t+1}) + \big\langle\pi_t(\cdot ~|~ x) - \pi_{t+1}(\cdot ~|~ x), \nabla L_{x,t}(\pi)|_{\pi=\pi_{t+1}}\big\rangle + B_{L_{x,t}}(\pi_{t}(\cdot | x) ~||~ \pi_{t+1}(\cdot | x)) \notag\\
        &\le L_{x,t}(\pi_{t+1}) + B_{L_{x,t}}(\pi_{t}(\cdot | x) ~||~ \pi_{t+1}(\cdot | x))\notag\\
        &= L_{x,t}(\pi_{t+1}) - B_{\ell_{x}}(\pi_{t}(\cdot | x) ~||~ \pi_{t+1}(\cdot | x)) 
    \end{align}
    in which the first step is by the definition of Bregman divergence, the second step is because $\pi_{t+1}$ is the global maximizer of $L_{x,t}(\cdot)$ as we have mentioned in the proof of \cref{lem:no-regret-1}, and the last step is because $L_{x,t}(\pi) + \ell_{x}(\pi)$ is linear in $\pi$, which does not change the Bregman divergence.

    The above inequalities further give
    \begin{align}
        B_{\ell_{x}}(\pi_{t}(\cdot | x) ~||~ \pi_{t+1}(\cdot | x)) &\le L_{x,t}(\pi_{t+1}) - L_{x,t}(\pi_{t}) \notag\\
        &= L_{x,t-1}(\pi_{t+1}) - L_{x,t-1}(\pi_{t}) + \big\langle\pi_{t+1}(\cdot ~|~ x) - \pi_{t}(\cdot ~|~ x), f_t(x,\cdot)\big\rangle \notag\\
        &\le \big\langle\pi_{t+1}(\cdot ~|~ x) - \pi_{t}(\cdot ~|~ x), f_t(x,\cdot)\big\rangle, \label{eq:no-regret-3-2}
    \end{align}
    in which the last inequality $\pi_{t}$ is the global maximizer of $L_{x,t-1}(\cdot)$.

    On the other hand, by the Taylor expansion and the mean value theorem, we can write $B_{\ell_{x}}$ above as
    \begin{align}
        B_{\ell_{x}}(\pi_{t}(\cdot | x) ~||~ \pi_{t+1}(\cdot | x)) &= \frac{1}{2}\|\pi_t(\cdot ~|~ x) - \pi_{t+1}(\cdot ~|~ x)\|^2_{\nabla^2_\pi\ell_x (\pi_{t}')}, \label{eq:no-regret-3-1}
    \end{align}
    where $\pi_t' := \alpha\pi_t + (1-\alpha)\pi_{t+1}$ for some $\alpha\in[0,1]$ and $\nabla^2_\pi\ell_x (\pi_{t}')$ is the Hessian of $\ell_x(\pi)$ with respect to $\pi$ at $\pi_{t}'$.

    Finally, by the generalized Cauchy-Schwarz,
    \begin{align}
        &\quad\ \big\langle\pi_{t+1}(\cdot ~|~ x) - \pi_{t}(\cdot ~|~ x), f_t(x,\cdot)\big\rangle \notag\\
        &\le \|\pi_t(\cdot ~|~ x) - \pi_{t+1}(\cdot ~|~ x)\|_{\nabla^2_\pi\ell_x (\pi_{t}')}\|f_t(x,\cdot)\|_{(\nabla^2_\pi\ell_x)^{-1} (\pi_{t}')} \notag\\
        &\overset{\mathrm{(i)}}{=} \sqrt{2 B_{\ell_{x}}(\pi_{t}(\cdot | x) ~||~ \pi_{t+1}(\cdot | x))}\|f_t(x,\cdot)\|_{(\nabla^2_\pi\ell_x)^{-1} (\pi_{t}')} \notag\\
        &\le \sqrt{\frac{2}{\eta}B_{\ell_{x}}(\pi_{t}(\cdot | x) ~||~ \pi_{t+1}(\cdot | x))}\|f_t(x,\cdot)\|_{\infty} \notag\\
        &\overset{\mathrm{(ii)}}{\le} \sqrt{\frac{2}{\eta}\big\langle\pi_{t+1}(\cdot ~|~ x) - \pi_{t}(\cdot ~|~ x), f_t(x,\cdot)\big\rangle}\|f_t(x,\cdot)\|_{\infty} \notag\\
        &\overset{\mathrm{(iii)}}{\le}  R\sqrt{\frac{2}{\eta}\big\langle\pi_{t+1}(\cdot ~|~ x) - \pi_{t}(\cdot ~|~ x), f_t(x,\cdot)\big\rangle},
    \end{align}
    in which (i) is by \cref{eq:no-regret-3-1}, (ii) is by \cref{eq:no-regret-3-2}, and (iii) is by \cref{assum:reward}.

    This immediately implies 
    \begin{align}
        \big\langle\pi_{t+1}(\cdot ~|~ x) - \pi_{t}(\cdot ~|~ x), f_t(x,\cdot)\big\rangle \le \frac{2R^2}{\eta}. \label{eq:no-regret-3-3}
    \end{align}

    By \cref{lem:no-regret-2}, we have
    \begin{align}
        &\quad\ \sum_{t=1}^{T}\big\langle\pi(\cdot ~|~ x) - \pi_{t}(\cdot ~|~ x), f_t(x,\cdot)\big\rangle \notag\\
        &\le \sum_{t=1}^{T}\big\langle\pi_{t+1}(\cdot ~|~ x) - \pi_{t}(\cdot ~|~ x), f_t(x,\cdot)\big\rangle - \ell_x(\pi_1) + \ell_x(\pi)\notag\\
        &\le \frac{2R^2T}{\eta} + \eta D_{\mathrm{KL}}\big(\pi(\cdot~|~x)~||~\pi_{\rref}(\cdot~|~x)\big),
    \end{align}
    where the last step is by \cref{eq:no-regret-3-3}. 

    At this point, we have arrived at the advertised bound.
\end{proof}

\end{document}